\documentclass[letterpaper, 10pt, conference]{ieeeconf}      

\IEEEoverridecommandlockouts                              
\usepackage{amssymb, amsmath, latexsym, amsfonts, mathrsfs, amsbsy}
\usepackage[english]{babel}
\usepackage[dvips]{graphicx}
\graphicspath{./Images/}
\usepackage[cp1252, latin1]{inputenc}
\usepackage{epsfig, color}
\usepackage{cite}
\usepackage{listings}
\usepackage{setspace}
\usepackage{makeidx, glossaries}
\usepackage{url}

\usepackage[caption=false, font=footnotesize]{subfig}
\usepackage{mathbbol}
\usepackage{stmaryrd}
\usepackage{enumerate}
\usepackage{epigraph}
\usepackage{multirow}
\usepackage{footnote}
\usepackage{bbm}

\usepackage[T1]{fontenc}
\usepackage{calligra}

\usepackage[algo2e]{algorithm2e} 
\usepackage{lipsum}

\newtheorem{theorem}{Theorem}
\newtheorem{lemma}{Lemma}
\newtheorem{definition}{Definition}

\newtheorem{problem}{Problem}
\newtheorem{corollary}{Corollary}
\newtheorem{proposition}{Proposition}
\newtheorem{remark}{Remark}

\usepackage{physics}
\usepackage{color}
\usepackage{array,booktabs,calc}
\usepackage{algorithm}
\usepackage{algpseudocode}
\algnewcommand{\LineComment}[1]{\State \(\triangleright\) #1}
\makeatletter
\def\BState{\State\hskip-\ALG@thistlm}
\makeatother
\usepackage{float}
\usepackage{graphics} 
\usepackage{epsfig} 
\usepackage{mathptmx} 
\usepackage{times} 
\usepackage{multirow}
\usepackage{amsmath}
\usepackage{setspace}
\usepackage{amsbsy}
\usepackage{supertabular}
\usepackage{tkz-graph}
\usepackage{etex}
\usepackage{pgf, tikz}
\usepackage{bm}
\usetikzlibrary{shapes,shapes.geometric,arrows,fit,calc,positioning,automata,through,intersections}
\tikzset{diamond state/.style={draw,diamond}}

\usepackage{soul}

\algdef{SE}[SWITCH]{Switch}{EndSwitch}[1]{\algorithmicswitch\ #1\ \algorithmicdo}   {\algorithmicend\ \algorithmicswitch}%
\algdef{SE}[CASE]{Case}{EndCase}[1]{\algorithmiccase\ #1}{\algorithmicend\     \algorithmiccase}%
\algtext*{EndSwitch}%
\algtext*{EndCase}%

\algdef{SE}[EVENT]{Event}{EndEvent}[1]{\textbf{upon event}\ #1\ \algorithmicdo}{\algorithmicend\ \textbf{event}}%
\algtext*{EndEvent}

\DeclareMathOperator*{\argmin}{arg\,min}

\newcommand{\project}[2]{\operatorname{Pr}_{#1}(#2)}
\newcommand{\projectbig}[2]{\operatorname{Pr}_{#1}\left(#2\right)}
\newcommand{\proj}[1]{\operatorname{Pr}_{#1}}

\title{\LARGE 
Distributed and Proximity-Constrained \\C-Means for Discrete Coverage Control
}

\author{Gabriele~Oliva$^{1,*}$,~Andrea~Gasparri$^2$, Adriano~Fagiolini$^3$, and Christoforos N. Hadjicostis$^4$ 
\thanks{$^1$ Unit of Automatic Control, Department of Engineering, Universit\`a Campus Bio-Medico di Roma, via \'Alvaro del Portillo 21, 00128, Rome, Italy.\newline
$^2$ University  Roma Tre, Department of Engineering, Via della Vasca Navale 79, 00146, Rome, Italy. \newline
$^3$ Dipartimento di Energia, Ingegneria dell'Informazione e Modelli Matematici (DEIM), University of Palermo, Viale delle Scienze, Edificio 10, 90128, Palermo, Italy.\newline
$^4$ Department of Electrical and Computer Engineering, University of Cyprus, 75 Kallipoleos Avenue, P.O. Box 20537, 1678 Nicosia, Cyprus.\newline
 $^*$ Corresponding author, email: g.oliva@unicampus.it}\\
}

\begin{document}

\maketitle
\thispagestyle{empty}
\pagestyle{empty}

\begin{abstract}
In this paper we present a novel distributed coverage control framework for a network of mobile agents, in charge of covering a finite set of \emph{ points of interest} (PoI), such as people in danger, geographically dispersed equipment or environmental landmarks. 
The proposed algorithm is inspired by $C$-Means, an unsupervised learning algorithm originally proposed for non-exclusive clustering and for identification of cluster centroids  from a set of observations. To cope with the agents' limited sensing range and avoid infeasible coverage solutions, traditional $C$-Means needs to be enhanced with {\em proximity constraints}, ensuring  that each agent takes into account only neighboring PoIs.
The proposed coverage control framework provides useful information concerning the ranking or importance of the different PoIs to the agents, which can be exploited in further application-dependent data fusion processes, patrolling, or disaster relief applications.

\end{abstract}

\section{Introduction}
In the literature, \emph{coverage control} has been gaining particular attention by the scientific community in recent years.
Typical approaches  follow the path of the seminal works in  \mbox{\cite{cortes2003geometric,cortes2005coordination},} where agents interact in order to achieve two main objectives: (1) to partition the area of interest into zones for which a single agent is responsible for the covering (e.g., Voronoi regions); (2) to reach locations that minimize a function of the distance from each point in the region assigned to them (e.g., navigate to a weighted centroid of the assigned Voronoi region).
Over the years, the importance of taking into account agents' limited sensing and communication capabilities has become paramount \cite{laventall2009coverage,kantaros2016distributed}.
Recent innovations in the field include, among others, time varying methodologies \cite{lee2014multi}, dynamic approaches \cite{palacios2016distributed}, and adaptive management of agents' trust \cite{pierson2016adaptive}.
 Note that most of the existing literature focuses on continuous space coverage, while in several practical situations it might be required to cover just a discrete set of \emph{ points of interest} (PoI). In fact, a discrete space approach has lower computational demands with respect to continuous space ones and allows the modeling of specific targets of interest.
To the best of our knowledge, work at the state of the art focused on the  coverage of discrete points of interest is quite sparse. In a recent work \cite{Jiang:2017}, the authors use $k$-means to solve a discrete coverage problem where multiple agents are in charge of a set of PoIs.
A related work, even though not specifically focused on coverage, is \cite{montijano2014efficient} where roles are assigned to the agents via constrained discrete optimization.  In both cases, however, the agents are not constrained to have a limited sensing radius.

In this paper, we consider a distributed problem where there is a finite and discrete set of $n$ PoIs in fixed locations in space, and a set of $r<n$ mobile agents aims at achieving the following two objectives: (1) to partition the PoIs in $r$ groups such that PoIs in the same group are spatially close while PoIs in different groups are far apart; (2) to associate each agent to a group and assign a location to such agent in order to maximize the coverage.
Our distributed approach is inspired by the $C$-means \cite{Bezdek:1981}, a {\em non-exclusive} and unsupervised data clustering algorithm, where a set of points (e.g., the PoIs) are grouped so that each point can be associated, at the same time, to different clusters (i.e., to different agents) with different intensities. 
From a mathematical standpoint, we extend the classical \mbox{$C$-means} by introducing {\em proximity constraints}, which account for the limited sensing range of the agents.
As a matter of fact, results dealing with the distribution of clustering algorithms, including the C-means, can be found in the literature~\cite{Patel:2013,TMC:2014,cinesi,CDC2016Oliva}. Compared to these works, where the sensors collaborate in order to compute the centroids and their associations in a distributed fashion, in our work we reverse the perspective. That is, the agents play the role of the centroid and collaborate in order to compute their location and the association to the PoIs; the latter play the role of the sensors.

The outline of the paper is as follows: Section \ref{sec:prelim} provides some preliminary definitions, while Section \ref{sec:problemstatement} introduces the problem addressed in this paper; Sections \ref{sec:assignment} and \ref{sec:refinement} solve two sub-problems that are the basis for our algorithm, while Section \ref{sec:algorithm} provides the proposed algorithm along with simulation results; the final Section \ref{sec:conclusions} collects some conclusive remarks and future work directions.

\section{Preliminaries}
\label{sec:prelim}

In the following, we consider the Euclidean space $\mathbb{R}^d$, equipped with the usual norm. Let us define the {\em metric projection} onto a convex set as follows:
\begin{definition}[Metric Projection]
Let $X\subset \mathbb{R}^d$ be a convex set. The metric projection onto $X$ is a function  $\proj{X}:\mathbb{R}^d\rightarrow X$ such that
$
\project{X}{{\bf v}} = \argmin_{ {\bf z} \in X} \left \| {\bf v} - {\bf z} \right \|.
$
\end{definition}
In the following, we refer to a metric projection simply as projection.
To serve the scope of this paper, we now specialize the projection onto a closed ball.


\begin{definition}[Closed Ball]
A closed ball $\mathbb{B}_{i,\rho}\subset\mathbb{R}^d$ with radius $\rho$ and center ${\bf p}_i\in \mathbb{R}^d$ is given by
\mbox{$
\mathbb{B}_{i,\rho} =\{{\bf z}\in \mathbb{R}^d,\,:\, \|{\bf p}_i-{\bf z}\|\leq \rho\}
$}.
\end{definition}


\begin{definition}[Projection onto a Closed Ball]
Let \mbox{$\mathbb{B}_{i,\rho}\subset \mathbb{R}^d$} be the closed ball of radius $\rho>0$ centered at ${\bf p}_i\in \mathbb{R}^d$. The projection onto $\mathbb{B}_{i,\rho}$ is the function
\mbox{$
\project{\mathbb{B}_{i,\rho}}{{\bf v}} =  \alpha {\bf v} + (1-\alpha) {\bf p}_i
$},
where $\alpha=\rho/\| {\bf v}-{\bf p}_i\|$ if $\| {\bf v}-{\bf p}_i\|>\rho$ and $\alpha=1$ otherwise.
%
%
\end{definition}


We now review the basics of convex constrained optimization. The reader is referred to \cite{Zangwill} for a comprehensive overview of the topic.


\begin{definition}[Convex Constrained Optimization Problem]
A convex constrained optimization problem with $s$ constraints has the following structure:
\begin{equation}
\label{prob:convexproblem}
\begin{matrix}
\min_{{\bf x}\in \mathbb{R}^d} f({\bf x})\\[2pt]
\mbox{Subject to}\\
\begin{cases}
g_i({\bf x})\leq 0,&i=1,\ldots, q\\
h_i({\bf x})=0,&i=q+1,\ldots, s.
\end{cases}
\end{matrix}
\end{equation}
where $f$ and all $g_i$ and $h_i$ are convex functions. 
The set of feasible solutions for the above problem is given by
$$\Omega=\{{\bf x}\in \mathbb{R}^d\,|\,g_i({\bf x})\leq 0,\,\, 1\leq i\leq q;\, h_i({\bf x})= 0,\,\,q+1\leq i\leq s\},$$
which is a convex set since all $g_i$ and $h_i$ are convex.
\end{definition}


We now review conditions that are fundamental in order to solve convex constrained optimization problems.


\begin{definition}[Slater's Conditions]
The convex constrained optimization problem in Eq.~\eqref{prob:convexproblem} satisfies the {\em Slater's Condition} if there is an ${\bf x}\in \Omega$ such that $g_i({\bf x})<0$ for all $i=1, \ldots, q$ and $h_i({\bf x})=0$ for all $i=q+1, \ldots, s$.
%
Moreover, if there is an ${\bf x}\in \Omega$ that satisfies the Slater's Condition and is such that all nonlinear $g_i({\bf x})$ are negative, while all linear $g_i({\bf x})$ are non-positive, then the problem is said to satisfy the {\em restricted Slater's Condition}.
\end{definition}


We now review Karush-Kuhn-Tucker (KKT) Theorem (see \cite{Zangwill} and references therein).


\begin{theorem}[KKT Theorem]
\label{theoKKT}
Consider a convex constrained optimization problem as in Eq.~\eqref{prob:convexproblem} that satisfies the restricted Slater's Condition, and let  the {\em Lagrangian function} be the function
\mbox{$
f({\bf x},{\bm \lambda})=f({\bf x})+ \sum_{i=1}^q \lambda_i g_i({\bf x})+
 \sum_{i=q+1}^s \lambda_i h_i({\bf x}),
$}
where ${\bm \lambda}=[\lambda_1,\ldots, \lambda_s]^T$.
The point ${\bf x}^*\in \mathbb{R}^d$ is a {\em global minimizer} for the problem if and only if the following conditions hold:
1)~\mbox{$\partial f({\bf x},{\bm \lambda})/\partial {{\bf x}} |_{{\bf x}^*,{\bm \lambda}^*}=0$}; 
2)~$\lambda^*_i g_i({\bf x}^*)=0$, for all $i=1,\ldots,q$;
3)~$g_i({\bf x}^*)\leq 0,\, 1\leq i\leq q$ and $h_i({\bf x}^*)=0,\, q+1\leq i \leq s$;
4)~\mbox{$\lambda^*_{i}\geq 0$}, for all $i=1,\ldots,s$.
\end{theorem}

\section{Problem Statement}
\label{sec:problemstatement}
Let us consider a set of $n$ PoIs and \mbox{$r< n$} agents deployed in a bounded region in $\mathbb{R}^d$, with \mbox{$d \in \{2, \, 3\}$}; we assume that each PoI $i$ is in a fixed location ${\bf p}_i\in\mathbb{R}^d$,  while the $j$-th agent is initially in location ${\bf x}_j(0)\in \mathbb{R}^d$.
The following set of assumptions is now in order: a)~ an agent $j$ is able to sense all PoIs for which the Euclidean distance from $j$ is less than or equal to~$\rho$; b)~an agent $j$ is able to communicate with all other agents for which the Euclidean distance from $j$ is less than or equal to a threshold $\theta\geq 2\rho$;
c)~the agents are deployed initially in a way such that each PoI is sensed by at least one agent, and each agent senses at least one PoI; d)~the set $\{{\bf p}_1,\ldots, {\bf p}_n\}$ contains at least $n>r$ distinct points.
Note that the last two assumptions, i.e.,  c) and d), are inherited from the standard C-means \cite{Bezdek:1981}.
In particular,  Assumption c) ensures that no PoI or agent is neglected by the optimization process.

Let ${\bf x}=[{\bf x}_1^T,\ldots, {\bf x}_r^T]^T\in \mathbb{R}^{rd}$ collect the location of all agents and let $U$ be an $n\times r$ matrix encoding the associations among the agents and the PoIs, having $u_{ij}$ at its $(i,j)-th$ entry. Moreover, let $\delta_{ij}=\|{\bf p}_i-{\bf x}_j\|$.  The problem can be stated as follows.


\begin{problem}
\label{prob:ourproblem}
Find \mbox{${\bf x}^*\in\mathbb{R}^{rd}$} and $U^*\in\mathbb{R}^{n\times r}$ such that
$$
\begin{matrix}
J({\bf x}^*,U^*)=\min_{{\bf x},U} \sum_{i=1}^n \sum_{j=1}^r u_{ij}^m\delta^2_{ij}\\
\mbox{Subject to}\\
\begin{cases}
\begin{matrix}
\sum_{j=1}^r u_{ij} =1,& \forall i& (I)\\
\sum_{i=1}^n u_{ij}>0,& \forall j& (II)\\
u_{ij}\in [0,1],& \forall i,j& (III)\\
u_{ij}(\delta_{ij}^2-\rho^2)\leq0,& \forall i,j& (IV)
\end{matrix}
\end{cases}
\end{matrix}
$$
where $m$ is a finite positive integer that is greater than one. 
\end{problem}


Problem~\ref{prob:ourproblem} extends the standard C-means problem \cite{Bezdek:1981} through the \emph{ proximity constraints}  (IV). 
Note that constraint (I) ensures the complete coverage of each PoI by the agents, whereas constraint (II) ensures that each agent covers at least one PoI. We point out that the larger the parameter  $m$ is, the less the associations $u_{ij}$ influence the objective function; as in the case of C-means, such a parameter accounts for the degree of overlapping among the different clusters\footnote{When there is no a priori knowledge, a popular choice is $m=2$ (see \cite{nayak2015fuzzy} and references therein).}. 

Note also that constraint (III)  requires \mbox{$u_{ij}\in[0,1]$}; as a consequence, the combination of constraint (III)  and the proximity constraints (IV) implies that $u_{ij}=0$ whenever $\delta_{ij}>\rho$. Therefore, these constraints ensure that an agent~$j$ does not cover those PoIs that are out-of-reach from it.

In the following, we develop a distributed and iterative algorithm to solve Problem \ref{prob:ourproblem}; specifically, at each iteration $k$, our algorithm alternates between an {\em assignment phase} where we find the optimal associations $u_{ij}(k)$ given the current locations ${\bf x}(k)$ and a {\em refinement phase} where we assume the associations $u_{ij}(k)$ are fixed and we find the optimal location ${\bf x}(k+1)$ for the agents\footnote{We  abstract away from the actual kinematics of the agents and we neglect the actual planning problem, assuming that a control loop to guide the agent toward the desired point exists.}.
The two steps are iterated until a stopping condition is met.

We next characterize (in Sections \ref{sec:assignment} and \ref{sec:refinement}) the optimality conditions for each sub-problem, respectively.
Future work will be focused on studying the optimality of the overall algorithm.

\section{Assignment Phase}
\label{sec:assignment}
In this section we discuss how, given fixed agents' locations, the optimal associations $U$ can be computed.

Specifically, the optimal choice for the $i$-th row of $U$ (i.e., the associations of the $i$-th PoI to all the agents) is as follows. If there is at least one agent $j$ such that ${\bf x}_j = {\bf p}_i$ (i.e., the $j$-th agent lies in the same location as the $i$-th PoI), the terms $u_{i1}, \ldots, u_{i,r}$  must satisfy\footnote{When two or more agents are at the same position of a poi ${\bf p}_j$, there is an infinity of choices for the associations $u_{ij}$ that satisfy Eq.~\eqref{calcolamembership2}.}
\begin{equation}
\label{calcolamembership2}
\sum_{j=1}^r u_{ij}(k)=1,\quad u_{ij}(k)=0 \mbox{ if } \delta_{ij}\neq 0;
\end{equation}
if, conversely, ${\bf x}_j \neq {\bf p}_i$ for all agents $j$ (i.e., no agent is at the same location as the $i$-th PoI), the terms $u_{ij}$ are as follows:
\begin{equation}
\label{calcolamembershipsparse}
u_{ij}=\begin{cases}\left({\underset{h|\delta_{ih}\leq\rho}{\sum}\Big ( \frac{\delta_{ij}}{\delta_{ih}}\Big ) ^{\frac{2}{m-1}}}\right)^{-1},& \mbox{ if } \delta_{ij}\leq\rho,\\
0,& \mbox{ otherwise}.
\end{cases}
\end{equation}

We now prove the optimality of such strategy.


\begin{theorem}
\label{prop:propositionassignment}
For given agents' locations ${\bf x}$ that satisfy Assumption c), a solution $\{{\bf x}, U\}$ where $U$ satisfies Equations \eqref{calcolamembership2} and \eqref{calcolamembershipsparse} is a  global minimizer for Problem~\ref{prob:ourproblem}.
\end{theorem}

\begin{proof} 
Consider a given ${\bf x}$ that satisfies Assumption c). Problem~\ref{prob:ourproblem} reduces to finding $U$ that minimizes \mbox{$J({\bf x}, U)=J(U)$} subject to constraints (I)--(III); as for the proximity constraints (IV) we have that if $\delta_{ij}\leq \rho$ then the corresponding proximity constraint is not active, while for $\delta_{ij}> \rho$, since $u_{ij}\in[0,1]$, the constraint reduces to $u_{ij}=0$.
As a consequence, the objective function becomes
$$
J(U)= \sum_{i=1}^n \sum_{j|\delta_{ij}\leq\rho} u_{ij}^m\delta^2_{ij}.
$$
Let $I_1$ be the index set of the PoIs that are not overlapped by any agent and $I_2$ be the complementary set, i.e., \mbox{$I_1=\{i\in\{1,\ldots,n\}\,|\,\delta_{ij}>0$} for all $j\in\{1,\ldots,r\}$ and \mbox{$I_2=\{1,\ldots,n\}\setminus I_1$}.
We can split $J(U)$ in two terms
\begin{equation}
\label{eq:thetasplit}
J(U)= \underbrace{\sum_{i\in I_1} \sum_{j| \delta_{ij}\leq\rho} u_{ij}^m\delta^2_{ij}}_{J_1(U)}+ \underbrace{\sum_{i\in I_2} \sum_{j| \delta_{ij}\leq\rho} u_{ij}^m\delta^2_{ij}}_{J_2(U)}.
\end{equation}
From Eq.~\eqref{calcolamembership2}, it follows that all terms $\delta_{ij}$ appearing in $J_2(U)$ are null, hence $J_2(U)=0$. This implies that the optimality of $U$ only depends on the entries $u_{ij}$ appearing in $J_1(U)$.
Since ${\bf x}$ satisfies Assumption c) by construction, using Eq.~\eqref{calcolamembershipsparse} it follows that constraints (II) and (III) are always satisfied. 
Furthermore, it is convenient to rewrite the variables $u_{ij}$ as the square of new variables $w_{ij}$, i.e., \mbox{$u_{ij}=w^2_{ij}$.} By doing this, and since constraint (IV) implies that $u_{ij}=0$ when $\delta_{ij}>\rho$, we can rewrite constraint (I) as
\mbox{$
\sum_{j| \delta_{ij}\leq\rho} w^2_{ij}=1
$}.
As a result, our problem becomes that of finding $W^*\in \mathbb{R}^{r\times d}$ that solves
$$
\begin{matrix}
J(W^*)=\min_{W} \sum_{i\in I_1} \sum_{j| \delta_{ij}\leq\rho} w_{ij}^{2m}\delta^2_{ij}\\
\mbox{Subject to}\\
\begin{cases}
\begin{matrix}
\sum_{j| \delta_{ij}\leq\rho} w^2_{ij}=1,& \forall i\in I_1.&
\end{matrix}
\end{cases}
\end{matrix}
$$
Note that the above problem is convex and has just equality constraints, hence the Slater's Condition is trivially satisfied and Theorem~\ref{theoKKT} applies.

Let ${\bm \zeta}=[\zeta_1,\ldots \zeta_q]^T$ be the Lagrangian multipliers associated to the constraints, while the corresponding Lagrangian function is
$$
 J(W,{\bm \zeta})=\sum_{i\in I_1} \sum_{j| \delta_{ij}\leq\rho} w^{2m}_{ij}\delta^2_{ij} + \sum_{i\in I_1} \zeta_i (
\sum_{j| \delta_{ij}\leq\rho} w^2_{ij} -1).
$$
By Theorem~\ref{theoKKT}, $W^*, {\bm \zeta}^*$ is a global optimal solution to the above problem if and only if it satisfies
\begin{equation}
\label{eq:partial2}
\frac{\partial J(W,{\bm \zeta})}{\partial w_{ij}}\Big|_{W^*,{\bm \zeta}^*}=
2m(w_{ij}^*)^{2m-1}\delta_{ij}+2w^*_{ij}\zeta^*_i=0,
\end{equation}
for all $i\in I_1$ and $j\in\{1,\ldots, r\}$ and
\begin{equation}
\label{eq:partial1}
\sum_{j| \delta_{ij}\leq\rho} (w^*)^2_{ij} =1,\quad \forall i\in I_1.
\end{equation}
From Eq. \eqref{eq:partial2} it follows that
\begin{equation}
\label{eq:wijsquareopt}
(w^*_{ij})^2=u_{ij}^*=\left(
{-\zeta^*_i}/{m\delta_{ij}^2}
\right)^{\frac{1}{(m-1)}}.
\end{equation}
Summing over all $h$ such that $\delta_{ih}\leq\rho$ and applying Eq. \eqref{eq:partial1} we get
\begin{equation}
\label{eq:wijsquareopt2}
(-\zeta_i^*)^{\frac{1}{(m-1)}}=\frac{1}{\sum_{h| \delta_{ih}\leq\rho}(
{1}/{m\delta_{ih}^2}
)^{\frac{1}{(m-1)}}}.
\end{equation}
By plugging Eq.~\eqref{eq:wijsquareopt2} in Eq.~\eqref{eq:wijsquareopt}, it follows that
$$
(w^*_{ij})^2= \frac{1}{(m\delta_{ij}^2)^{\frac{1}{(m-1)}}\sum_{h| \delta_{ih}\leq\rho}(
\frac{1}{m\delta_{ih}^2}
)^{\frac{1}{(m-1)}}}=
\frac{1}{\sum_{h| \delta_{ih}\leq\rho}(
\frac{\delta_{ij}}{\delta_{ih}}
)^{\frac{2}{(m-1)}}}.
$$
Therefore, since $u_{ij}^*=(w^*_{ij})^2$, we conclude that Eq. \eqref{calcolamembershipsparse} corresponds to the global optimal solution.
%
%
\end{proof}

\section{Refinement Phase}
\label{sec:refinement}
Let us now discuss the refinement phase within the proposed algorithm, i.e., a way to find the optimal location of the agents, for fixed associations. 

Note that, when the fixed associations are admissible for Problem~\ref{prob:ourproblem}, our problem is equivalent to solving a collection of independent sub-problems (i.e., one for each agent) having the following structure.


\begin{problem}
\label{prob:ourproblem2}
Find ${\bf x}^*_j\in \mathbb{R}^{d}$ that solves
\begin{equation}
\label{prob:singleagent}
\begin{matrix}
J({\bf x}^*_j)=\min_{{\bf x}_j\in \mathbb{R}^{d}} \underset{i| \delta_{ij}\leq\rho}{\sum} u_{ij}^m \delta_{ij}^2\\
\mbox{Subject to}\\
\begin{cases}
\delta_{ij}^2\leq \rho^2,\quad \forall i \mbox{ s.t. } u_{ij}>0.
\end{cases}
\end{matrix}
\end{equation}
\end{problem}


We now characterize the optimal solution of Problem~\ref{prob:ourproblem2}.
To this end, we first define the set of admissible solutions.


\begin{definition}[Admissible Solutions to Problem~\ref{prob:ourproblem2}]
\label{def:feasiblelocations}
The set of admissible solutions to Problem~\ref{prob:ourproblem2} is
\mbox{$
\mathbb{B}^*_j=\underset{i| u_{ij}>0}{\cap} \mathbb{B}_{i,\rho}
$}, where $\mathbb{B}_{i,\rho}$ is the ball of radius $\rho$ centered at the location ${\bf p}_i$ of the $i$-th PoI.
\end{definition}
\begin{remark}
\label{rem:isconvex}
Problem~\ref{prob:ourproblem2} is convex, since for fixed terms~$u_{ij}$ the objective function is a linear combination of convex functions, and similarly the constraints are convex functions.
\end{remark}


We now establish a condition for Problem~\ref{prob:ourproblem2} to satisfy  Slater's  Condition.


\begin{proposition}
\label{prop:isslater}
Problem~\ref{prob:ourproblem2}  satisfies Slater's Condition if and only if $\mathbb{B}^*_j$ is not a singleton.
\end{proposition}
\begin{proof}
The fact $\mathbb{B}^*_j$ is a singleton implies that at least two constraints are satisfied at the equality (i.e., at least two balls are tangent), thus preventing Slater's Condition to be satisfied. \end{proof}

We now provide an optimality condition to move an agent~$j$ when $U$ is fixed.


\begin{theorem}
\label{theo:optimalx}
Let $U$ be admissible to Problem~\ref{prob:ourproblem}, and suppose $\mathbb{B}^*_j$ is not a singleton.
Then
\begin{equation}
\label{eq:centrilambda}
{\bf x}^*_j= {\underset{i|\delta_{ij}\leq\rho}{\sum} (u_{ij}^m+\lambda^*_{i}){\bf p}_i}\,\,/\,{\underset{i|\delta_{ij}\leq\rho}{\sum} (u_{ij}^m+\lambda^*_{i}) }
\end{equation}
\end{theorem}
is a global minimizer for Problem~\ref{prob:ourproblem2} 
if and only if  there exist  ${\bm \lambda}^*=[\lambda^*_1,\ldots,\lambda^*_n]^T\in \mathbb{R}^n$  such that: (a)~${\bf x}^*_j\in \mathbb{B}^*_j$; (b)~$\lambda^*_i\geq 0$ for all $i$; (c)~$\lambda^*_{i}\left(\|{\bf p}_i-{\bf x}^*_j\|^2-\rho^2\right)=0$ for all $i$.

\begin{proof}
As discussed in Remark~\ref{rem:isconvex}, Problem~\ref{prob:ourproblem2}  is convex. Moreover, as shown in Proposition~\ref{prop:isslater}, it satisfies Slater's condition and therefore it satisfies also the restricted Slater's condition. 
Therefore, the KKT Theorem (Theorem~\ref{theoKKT}) applies to Problem~\ref{prob:ourproblem2}.
Let us consider the Lagrangian function associated to Problem~\ref{prob:ourproblem2}, i.e.,
\begin{equation}
J({\bf x}_j,{\bm \lambda}^*)=\underset{i| \delta_{ij}\leq\rho}{\sum} u_{ij}^m \delta_{ij}^2+\underset{i| \delta_{ij}\leq\rho}{\sum}\lambda^*_{i}\left(\delta_{ij}^2-\rho^2\right).
\end{equation}
By Theorem~\ref{theoKKT}, a point ${\bf x}^*_j$ is a global minimizer for Problem~\ref{prob:ourproblem2} if and only if: 1)~$\partial J({\bf x}_j,{\bm \lambda})/ \partial {{\bf x}_j} |_{{\bf x}_j^*, {\bm \lambda}^*}=0$; 
2)~$\delta_{ij}\leq \rho$, for all $i\in\{1,\ldots,n\}$; 3)~$\lambda_{i}\geq0$, for all $i\in\{1,\ldots,n\}$; 4)~$\lambda^*_{i}\left(\|{\bf p}_i-{\bf x}^*_j\|^2-\rho^2\right)=0$, for all $i\in\{1,\ldots,n\}$.
Clearly, conditions (a)--(c) are equivalent to the above conditions (2)--(4).
To prove Theorem~\ref{theo:optimalx}, therefore, we must show that any solution satisfying condition (1) has the structure of Eq.~\eqref{eq:centrilambda}; to establish this, we show that, for any ${\bm \lambda}$,
$\partial J({\bf x}_j,{\bm \lambda})/ \partial {{\bf x}_j}$
vanishes at ${\bf x}^*_j$  along any arbitrary direction ${\bf y}\in \mathbb{R}^d$ with $\|{\bf y}\|\neq0$.
Let $t\in \mathbb{R}$ and let ${\bf x}^*_j$ be the optimal solution for the problem in Eq. \eqref{prob:singleagent}, and let us define $J(t)=J({\bf x}^*_j+t{\bf y}, {\bm \lambda})$. It holds
\mbox{$
J(t)=\underset{i| \delta_{ij}\leq\rho}{\sum} (u_{ij}^m+\lambda_{i}) \|{\bf p}_i-({\bf x}^*_j+t{\bf y})\|^2-\rho^2\underset{i| \delta_{ij}\leq\rho}{\sum} \lambda_{i}
$}.
We can rearrange $J(t)$ as
\begin{equation*}
J(t)=\underset{i| \delta_{ij}\leq\rho}{\sum} (u_{ij}^m+\lambda_{i})({\bf p}_i-{\bf x}^*_j-t{\bf y})^T({\bf p}_i-{\bf x}^*_j-t{\bf y})-\rho^2\underset{i| \delta_{ij}\leq\rho}{\sum} \lambda_{i},
\end{equation*}
so that
$
d J(t) / d t=-2\underset{i| \delta_{ij}\leq\rho}{\sum} (u_{ij}^m+\lambda_{i}){\bf y}^T ({\bf p}_i-{\bf x}^*_j-t{\bf y})
$
and the KKT condition~(1) is satisfied if
\begin{equation}
\label{eq:dvJ}
 \begin{aligned}
\dv{J(0)}{t}
&=-2\underset{i| \delta_{ij}\leq\rho}{\sum} (u_{ij}^m+\lambda_{i}) {\bf y}^T ({\bf p}_i-{\bf x}^*_j)=\\
&=-2 {\bf y}^T \underset{i| \delta_{ij}\leq\rho}{\sum} (u_{ij}^m+\lambda_{i})({\bf p}_i-{\bf x}^*_j)=0.
\end{aligned}   
\end{equation}
For arbitrary nonzero ${\bf y}$, Eq.~\eqref{eq:dvJ} holds
if and only if 
$$\underset{i| \delta_{ij}\leq\rho}{\sum} (u_{ij}^m+\lambda_{i})({\bf p}_i-{\bf x}^*_j)=0,$$ 
which, considering ${\bm \lambda}={\bm \lambda}^*$, is equivalent to Eq. \eqref{eq:centrilambda}.
This concludes the proof.  
\end{proof}


We now present some technical results that are fundamental in order to construct a global minimizer for Problem~\ref{prob:ourproblem2}.


\begin{lemma}
\label{lem:projectioninhull}
Given a nonempty, non-singleton set $\mathbb{B}^*$, obtained as the intersection of a collection of $n$ balls $\{\mathbb{B}_{1,\rho},\ldots,\mathbb{B}_{n,\rho}\}$, each centered at ${\bf v}_i \in \mathbb{R}^d$ and with radius $\rho$, i.e.
\mbox{$
\mathbb{B}^*=\cap_{i=1}^n \mathbb{B}_{i,\rho}
$}, for every point ${\bf v}\in \mathbb{R}^d$ there exist non-negative $\mu_1,\ldots, \mu_n$ and a positive $\overline \mu$ with \mbox{$\overline \mu+\sum_{i=1}^n \mu_i =1$}, such that the projection of ${\bf v}$ onto $\mathbb{B}^*$ is given by \mbox{$
\project{\mathbb{B}^*}{{\bf v}}= \overline \mu {\bf v} + \sum_{i=1}^n \mu_i {\bf v}_i
$}.
\end{lemma}

\begin{proof}
To prove our statement, we show that $\project{\mathbb{B}^*}{{\bf v}}$ is the solution of a convex optimization problem, which can be solved by resorting to KKT Theorem.

Let us recall that the projection $\project{\mathbb{B}^*}{{\bf v}}$ is the point of minimum distance from ${\bf v}$ among those in $\mathbb{B}^*$. In other words, $\project{\mathbb{B}^*}{{\bf v}}$ is the solution of the following problem
$$
\begin{matrix}
\project{\mathbb{B}^*}{{\bf v}}=\min_{{\bf z}\in \mathbb{R}^d} \| {\bf z}-{\bf v}\|^2\\
\mbox{Subject to}\\
\begin{cases}
\begin{matrix}
\| {\bf z}-{\bf v}_i\|^2\leq \rho^2,& \forall i=1,\ldots,n
\end{matrix}
\end{cases}
\end{matrix}
$$
The above problem is convex, since $\| {\bf z}-{\bf v}\|^2$ and $\| {\bf z}-{\bf v}_i\|^2$ are convex functions.
Moreover, since we assumed $\mathbb{B}^*$ is not empty nor a singleton, the above problem satisfies Slater's Condition, so KKT Theorem~\ref{theoKKT} applies.

Let $\gamma_i$ be the Lagrangian multiplier associated to the \mbox{$i$-th} constraint, and let ${\bm \gamma}$ be the stack vector of all $\gamma_i$. The Lagrangian function $L({\bf z},{\bm \gamma})$ for the above problem is
\mbox{$
L({\bf z},{\bm \gamma})= \| {\bf z}-{\bf v}\|^2 + \sum_{i=1}^n \gamma_i (\| {\bf z}-{\bf v}_i\|^2- \rho^2)
$};
hence, \mbox{${\partial L({\bf z},{\bm \gamma})}/{\partial {\bf z}}|_{{\bf z}^*, {\bm \gamma}^*}= 2{\bf z}^*-2{\bf v} + 2\sum_{i=1}^n \gamma^*_i ({\bf z}^*-{\bf v}_i)=0$}, which implies that the optimal solution has the form 
\begin{equation}
\label{eq:gamma}
{\bf z}^*=\frac{{\bf v}+\sum_{i=1}^n \gamma^*_i {\bf v}_i }{1+\sum_{i=1}^n \gamma^*_i},
\end{equation}
where the Lagrangian multipliers $\gamma^*_i$ must satisfy $\gamma^*_i\geq 0$ for all $i=1,\ldots, n$ and  $\gamma^*_i (\| {\bf z}^*-{\bf v}_i\|^2- \rho^2)=0$  for all $i=1,\ldots, n$.
Under the above constraints on $\gamma^*_i$, our statement is verified for
\mbox{$\overline \mu = {1}/{(1+\sum_{i=1}^n \gamma^*_i)}>0$}
and
\mbox{$\mu_i = {\gamma^*_i}/{1+\sum_{i=1}^n \gamma^*_i}\geq 0$}, for all $i=1,\ldots, n$.
This completes our proof. \end{proof}


\begin{corollary}
\label{coroll:good}
If ${\bf v}\not \in \mathbb{B}^*$ then $\project{\mathbb{B}^*}{{\bf v}}$ lies on the intersection of the boundaries of the balls corresponding to positive coefficients $\gamma_i$.
\end{corollary}
\begin{proof}
From Lemma~\ref{lem:projectioninhull} the coefficients $\gamma_i$ are non-negative and must satisfy
\mbox{$
\gamma_i (\| {\bf z}-{\bf v}_i\|^2- \rho^2)=0$}, for all \mbox{$i=1,\ldots, n$}.
Hence, $\project{\mathbb{B}^*}{{\bf v}}$ must belong to the intersection of the boundaries of the balls associated to positive $\gamma_i$.
The proof is complete noting that, when \mbox{${\bf v}\not \in \mathbb{B}^*$},  Eq.~\eqref{eq:gamma} implies that there must be at least a positive $\gamma_i$.
 \end{proof}
 
 
We now provide a technical result which will be used later in the main theorem of this section.
 
 
\begin{lemma}
\label{lem:sherman}
Let $V=\{{\bf v}_1,\ldots, {\bf v}_n\}\subset \mathbb{R}^d$ with $n>1$ and 
consider given $\overline \mu>0$ and $\mu_i \in (0,1)$ for all $i\in\{1,\ldots,n\}$ such that
\mbox{$\overline \mu+\sum_{i=1}^n \mu_i=1$}.
For any choice of $\alpha>0$ there is a choice of $\lambda_i$ for all $i\in\{1,\ldots,n\}$ satisfying
\begin{equation}
\label{eq:choiceoflambda1}
\mu_i={\lambda_i}/{(\alpha +  \sum_{h=1}^n \lambda_h)},\quad \lambda_i\geq 0.
\end{equation}
\end{lemma}
\begin{proof}
For the sake of the proof, it is convenient to rearrange Eq.~\eqref{eq:choiceoflambda1} as
\mbox{$\lambda_i=\mu_i(\alpha +  \sum_{h=1}^n \lambda_h)$}, $\lambda_i\geq 0$.
Stacking for all $\lambda_i$ we get
\begin{equation}
\label{eq:choiceoflambda2}
{\bm \lambda}=\alpha {\bm \mu}+ {\bm \mu} {\bf 1}^T {\bm \lambda},
\end{equation}
where ${\bm \lambda}$ and ${\bm \mu}$ are the stack vectors collecting all $\lambda_i$ and $\mu_i$, respectively.
Since by assumption $\overline \mu>0$, it holds \mbox{$1-{\bf 1}^T{\bm \mu}=\overline \mu>0$}, from the Sherman-Morrison Theorem \cite{sherman1950adjustment} it follows that the matrix 
\mbox{$(I-{\bm \mu}{\bf 1}^T)^{-1}$} exists and has the following structure
\begin{equation}
\label{eq:intermediatesherman2a}
(I-{\bm \mu}{\bf 1}^T)^{-1}=
I+{{\bm \mu}{\bf 1}^T}/{\overline \mu},
\end{equation}
where we notice that all entries are non-negative by construction.
At this point, by plugging Eq.~\eqref{eq:intermediatesherman2a} in Eq.~\eqref{eq:choiceoflambda2} we obtain
\mbox{${\bm \lambda}=(I-{\bm \mu}{\bf 1}^T)^{-1}\alpha{\bm \mu}=(I+{{\bm \mu}{\bf 1}^T}/{\overline \mu})\alpha{\bm \mu}$}.
Since $\alpha>0$,  $\mu_i\geq 0$ and the matrix in Eq.~\eqref{eq:intermediatesherman2a} has non-negative entries, we conclude that ${\bm \lambda}\geq0$. \end{proof}

We now provide a constructive method to obtain a solution to Problem~\ref{prob:ourproblem2}. 

\begin{theorem}
\label{prop:propositionrefinement}
Let $U$ be admissible to Problem~\ref{prob:ourproblem} and assume that ${\bf x} \in \mathbb{R}^{rd}$ satisfies Assumption c).
Then
\begin{equation}
\label{fuzzycenterssparsefinal}
{\bf x}^*_j=\projectbig{\mathbb{B}^*_j}{{\underset{i|\delta_{ij}\leq\rho}{\sum} u^m_{ij} {\bf p}_i}\,/{\underset{i|\delta_{ij}\leq\rho}{\sum} u^m_{ij} }}
\end{equation}
is a global minimizer for Problem~\ref{prob:ourproblem2}.
\end{theorem}
\begin{proof} 
Let us define the set $I_{j}=\{i\in\{1,\ldots,n\}\,|\, u_{ij}>0\}$ and let
\begin{equation}
\label{eq:hatxhatudef}
\hat {\bf x}_j= {\sum_{i=1}^n u_{ij}^m {\bf p}_i }\,/\,{\sum_{i=1}^n u_{ij}^m }, \quad \hat u=\sum_{i=1}^n u_{ij}^m.
\end{equation}
We first handle two special cases. 
Note that, if $\hat {\bf x}_j \in \mathbb{B}^*_j$ then $\project{\mathbb{B}^*_j}{\hat {\bf x}_j}=\hat {\bf x}_j$ and $\hat {\bf x}_j$ itself satisfies Theorem~\ref{theo:optimalx} with all $\lambda^*_i=0$.
Moreover, if $\mathbb{B}^*_j$ is a singleton, since ${\bf x}$ satisfies Assumption c) it must hold $\mathbb{B}^*_j=\{{\bf x}_j\}$. In this case, Problem~\ref{prob:ourproblem2} fails to satisfy Slater's conditions and Theorem \ref{theo:optimalx} does not apply;
however, Theorem \ref{theo:optimalx} is no longer required if $\mathbb{B}^*_j=\{{\bf x}_j\}$, as $\project{\mathbb{B}^*_j}{\hat {\bf x}_j}={\bf x}_j$ by construction. In other words, in this case agent~$j$ does not move.
We now focus on the case $\hat {\bf x}_j \not\in \mathbb{B}^*_j$ and $\mathbb{B}^*_j\neq \{{\bf x}_j\}$, that is ${\bf x}_j$ does not belong to $\mathbb{B}^*_j$ and $\mathbb{B}^*_j$ is not a singleton.
In this setting, our goal is to show that $\project{\mathbb{B}^*_j}{\hat {\bf x}}$ is a solution having the form of Eq.~\eqref{eq:centrilambda}, for which a closed form of the Lagrangian multipliers $\lambda^*_i$ that satisfies conditions (a)--(c) in Theorem~\ref{theo:optimalx} is given.
If  $\hat {\bf x}\not\in \mathbb{B}^*_j$, then $\project{\mathbb{B}^*_j}{\hat {\bf x}}\in \mathbb{B}^*_j$ lies on the boundary of $\mathbb{B}^*_j$, hence Condition (a) in Theorem~\ref{theo:optimalx} is satisfied.
By Lemma~\ref{lem:projectioninhull} and Corollary~\ref{coroll:good}, there is an $I^*_j\subseteq I_j$ such that
\begin{equation}
\label{eq:bcomb1}
\project{\mathbb{B}^*_j}{\hat {\bf x}}= \hat \mu \hat {\bf x} + \sum_{i\in I^*_j} \mu_i {\bf p}_i.
\end{equation}
with $\hat \mu\in(0,1)$, $\mu_i\in(0,1)$ for all $i\in I^*_j$ and \mbox{$\hat \mu+ \sum_{i\in I^*_j} \mu_i=1$.}
In other words, the projection is a convex combination of $\hat {\bf x}$ and the location of the PoIs ${\bf p}_i$ for $i\in I^*_j$, where the coefficient $\hat \mu$ for $\hat {\bf x}$ is strictly positive by construction.

Let ${\bm \lambda}^*$ and ${\bm \mu}$ be the stack vectors of the Lagrangian multipliers $\lambda^*_i$ and the coefficients $\mu_i$ for all $i\in I^*_j$, respectively. By Proposition~\ref{lem:sherman}, choosing
\mbox{${\bm \lambda}^*=\hat u(I+{{\bm \mu}{\bf 1}^T}/{\hat \mu}){\bm \mu}\geq0$}
implies
\begin{equation}
\label{eq:choiceoflambda}
\mu_i ={\lambda^*_i}/{(\hat u +  \sum_{h\in I^*_j } \lambda^*_h)}, \quad \forall i\in I^*,
\end{equation}
and therefore
\begin{equation}
\label{eq:choiceoflambdaa}
\hat \mu={\hat u}/{(\hat u +  \sum_{h\in I^*_j } \lambda^*_h)}.
\end{equation}
Plugging Eq.~\eqref{eq:choiceoflambda} and Eq.~\eqref{eq:choiceoflambdaa} in Eq.~\eqref{eq:bcomb1}, and choosing $\lambda^*_i=0$ for all $i\in \{1,\ldots,n\}\setminus I^*_j$ we get
\mbox{$
\project{\mathbb{B}^*_j}{\hat {\bf x}}= \frac{\hat u}{\hat u +  \sum_{h\in I^*_j } \lambda^*_h} \hat {\bf x}_j + \sum_{i\in I^*_j} \frac{\lambda^*_i}{\hat u +  \sum_{h\in I^*_j } \lambda^*_h} {\bf p}_i
$},
which has the same structure as ${\bf x}^*_j$ in Eq.~\eqref{eq:centrilambda}.
Note that all $\lambda^*_i\geq 0$, hence condition (b) in Theorem~\ref{theo:optimalx} is satisfied. Moreover, as discussed above, Corollary~\ref{coroll:good} guarantees that 
$
\|\project{\mathbb{B}^*_j}{\hat {\bf x}}-{\bf p}_i\|=\rho$ for all $i\in I^*_j
$,
and since $\lambda^*_i=0$ for $i\not\in I^*_j$ also condition (c) in Theorem~\ref{theo:optimalx} is satisfied.
This completes the proof. \end{proof}

\section{Proposed Algorithm and Simulations}
\label{sec:algorithm}
We now discuss our distributed algorithm, based on the repeated alternated solution of the two sub-problems solved in the previous sections.
%
%
Specifically, each agent alternates between the assignment phase, where it computes the associations with the sensed PoIs based on its current location, and the refinement phase, where it selects a new location based on the sensed PoIs and the associations.
We point out that, knowing the associations, the refinement phase is executed locally by each agent, with no need for coordination among neighboring agents.
Conversely, communication among neighboring agents is required during the assignment phase, to collect all the distances $\|{\bf p}_i-{\bf x}_h(k)\|$ involving the $i$-th PoI and the agents able to sense it.
Note that communication is ensured by Assumption b).
This allows each agent to compute the associations via Eqs.~\eqref{calcolamembership2} ~and~\eqref{calcolamembershipsparse}.
Note also that the algorithm execution requires an implicit coordination among the agents, which can be achieved by resorting to well known protocols such as consensus algorithms \cite{Olfati1}.
Details are omitted for the sake of brevity.
Finally, a stopping criterion is required: a simple choice is to let each agent stop when its new location is closer than $\epsilon$ to the current one. 


Let us now numerically demonstrate the effectiveness of the proposed algorithm.
In Figure \ref{fig:example1}, we show an example where PoIs and agents are embedded in the unit square  $[0,1]^2\subset \mathbb{R}^2$. Specifically, we consider $n=140$ PoIs (circles) and $r=4$ agents (triangles).
The figure compares the result of the proposed algorithm for $\rho=0.35$ with the result of the standard C-means.
In Figures \ref{fig:example1:1} and \ref{fig:example1:2}, we assign the same colors to those PoIs having their largest association with the same agent. Moreover we report the initial and final position for the agents within the proposed algorithm in white and blue, respectively and, for comparison, we report in red the final positions found via standard C-means. According to the plots, in spite of sparsity, quite similar results are achieved.
The plots in Figures \ref{fig:example1:3} and \ref{fig:example1:4} show the associations obtained for one cluster within the proposed algorithm (for $\rho=0.35$) and within the standard C-means.
This gives an intuitive explanation of the effect of the proximity constraints on the associations, i.e., membership with a very low degree are truncated in the proposed algorithm.
Figure \ref{fig:example2} reports the value of the objective function with respect to the iterations for the solution found by the \mbox{C-means} and by the proposed algorithm for different values of~$\rho$.
Overall, the figure numerically demonstrate that the behavior of the proposed algorithm tends to coincide with the behavior of the standard C-means as $\rho$ grows, i.e., as $\rho$ increases, the agents become more and more aware of the PoIs, until the proximity constraints no longer exist. 
Notably, this numerical simulation suggests that our algorithm can be thought as a generalization of the standard C-means. 

\section{Conclusions}
\label{sec:conclusions}

In this paper we provide a novel distributed coverage algorithm for a set of  agents aiming at covering in a non-exclusive way a discrete and finite set of points of interest, which is inspired by the standard C-Means algorithm. 

Several directions are envisioned for future work: (1)~consider an asynchronous setting and moving PoIs; (2)~ introduce the possibility for an agent to leave some of the PoIs  uncovered if they are  already covered by other agents; (3)~provide a formal characterization of the convergence properties of the overall proposed algorithm.

\begin{figure}[ht!]
 \centering
\subfloat[]{\label{fig:example1:1} \includegraphics[width=0.21\textwidth]{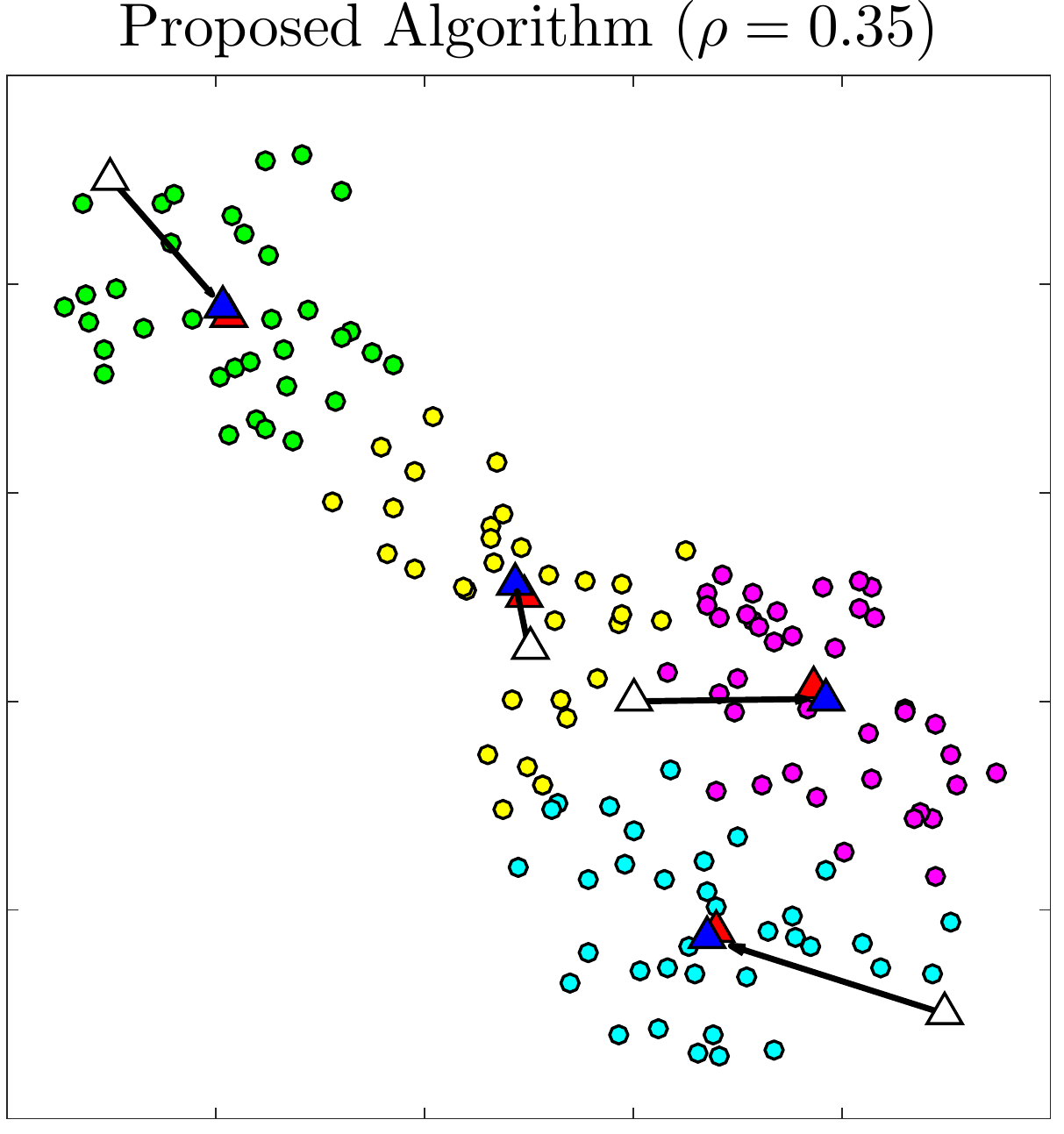}}
\hspace{1mm}
\subfloat[]{\label{fig:example1:2} \includegraphics[width=0.21\textwidth]{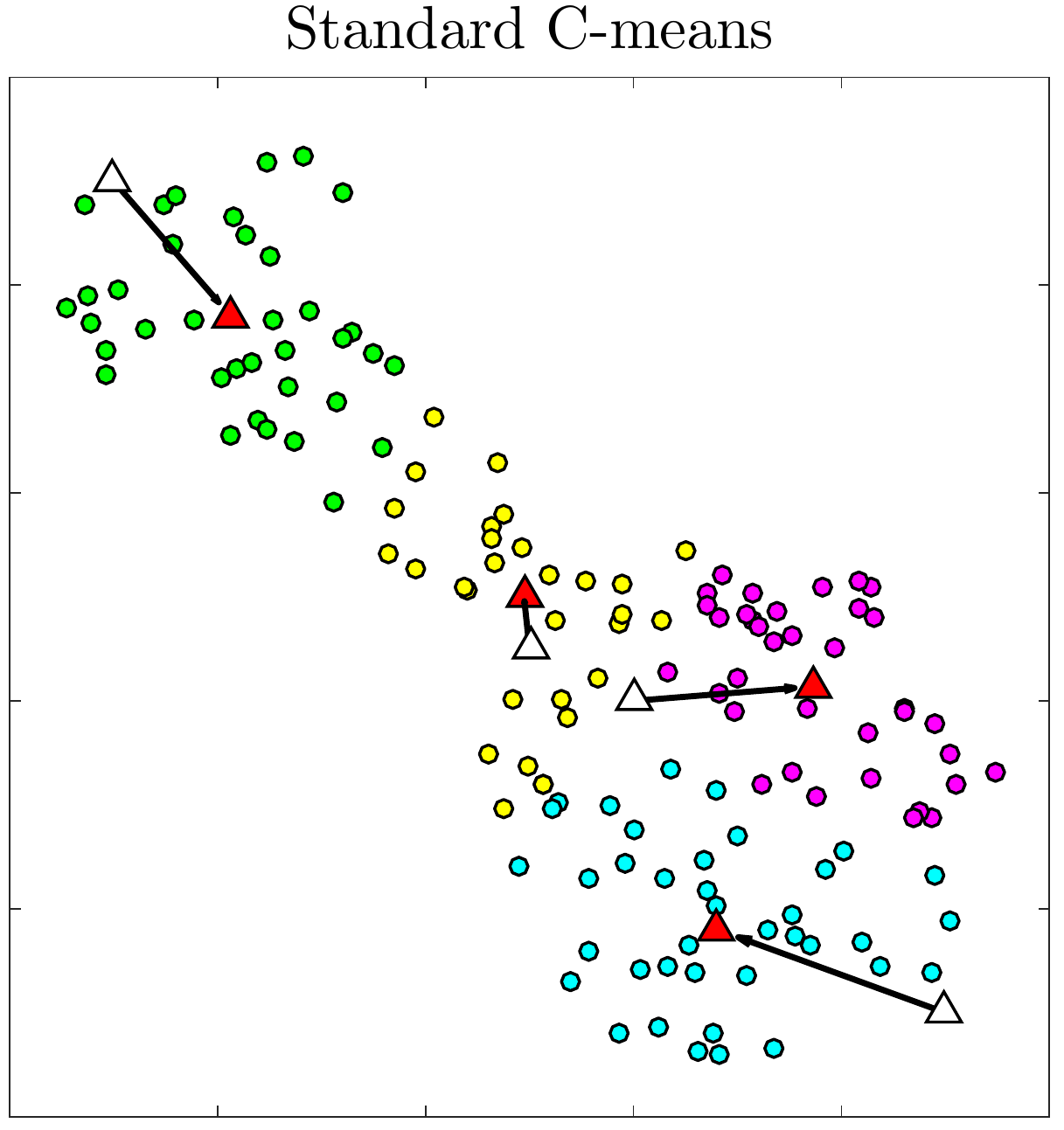}}
\vspace{1.5mm}
\subfloat[]{\label{fig:example1:3} \includegraphics[width=0.21\textwidth]{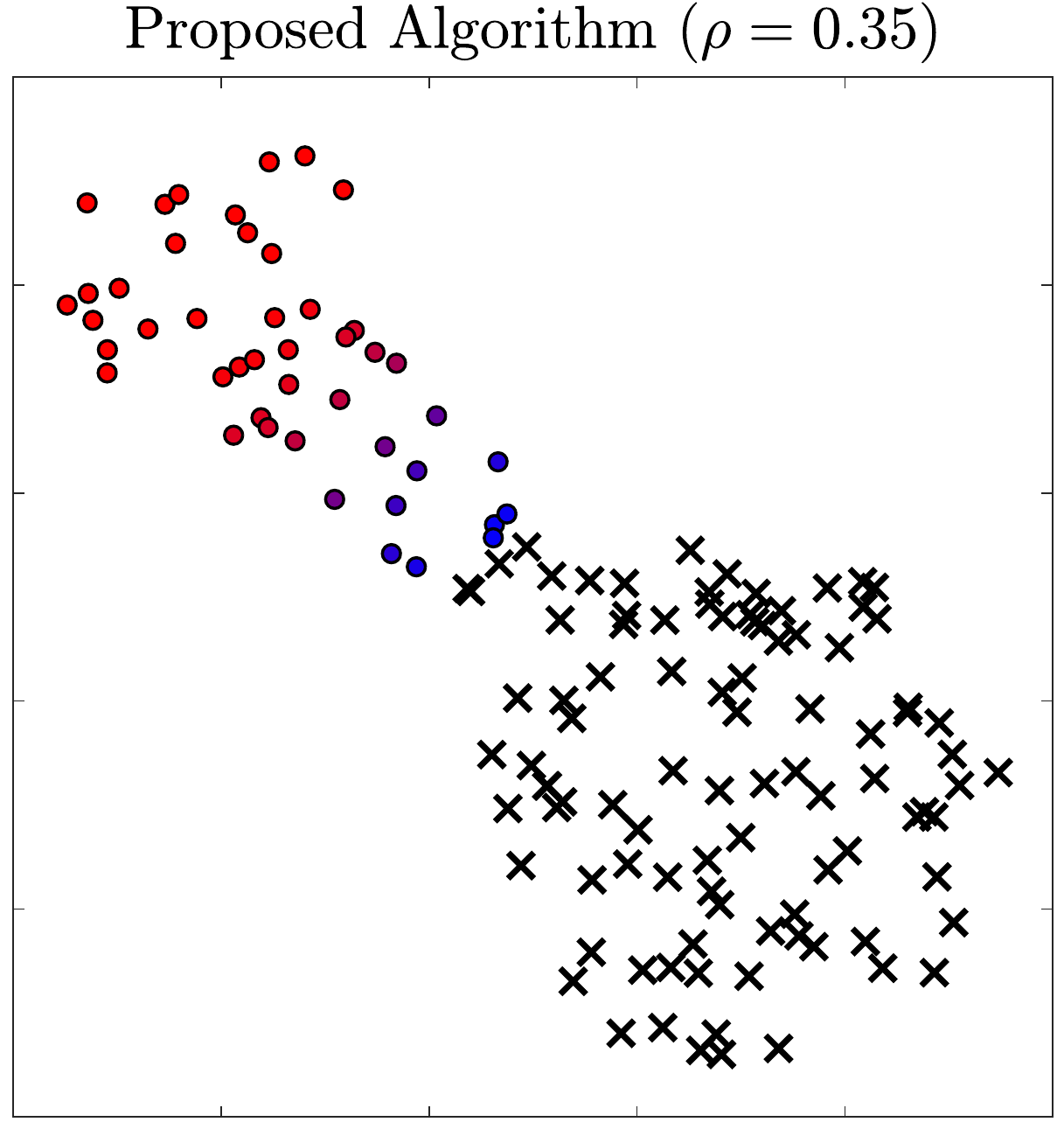}}
\hspace{1mm}
\subfloat[]{\label{fig:example1:4} \includegraphics[width=0.21\textwidth]{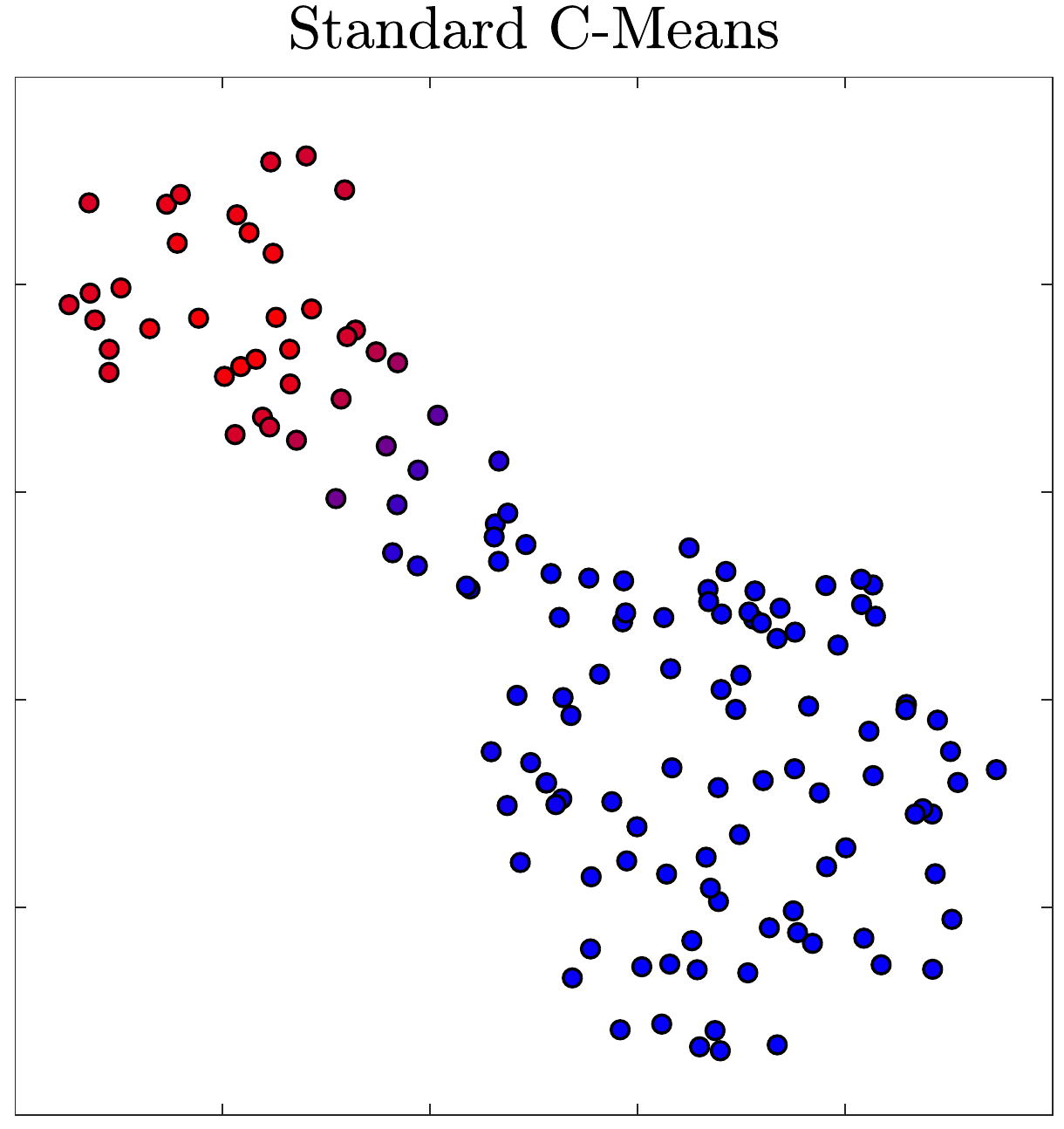}}
\caption{Plots \ref{fig:example1:1} and \ref{fig:example1:2}: comparison of the proposed algorithm for $\rho=0.35$ (plot \ref{fig:example1:1}) with the standard C-means (plot \ref{fig:example1:2}), considering a particular instance. 
Plots \ref{fig:example1:3} and \ref{fig:example1:4}: comparison of the associations between a particular agent and all the PoIs within the proposed algorithm when $\rho=0.35$ (plot \ref{fig:example1:3}) and within the standard C-means (plot \ref{fig:example1:4}).
The color of the PoI tends to red or blue as the associations approach one or zero, respectively. Black crosses indicate associations that are exactly zero.}
 \label{fig:example1}
\end{figure}

 \begin{figure}
 \centering
\includegraphics[width=.45\textwidth]{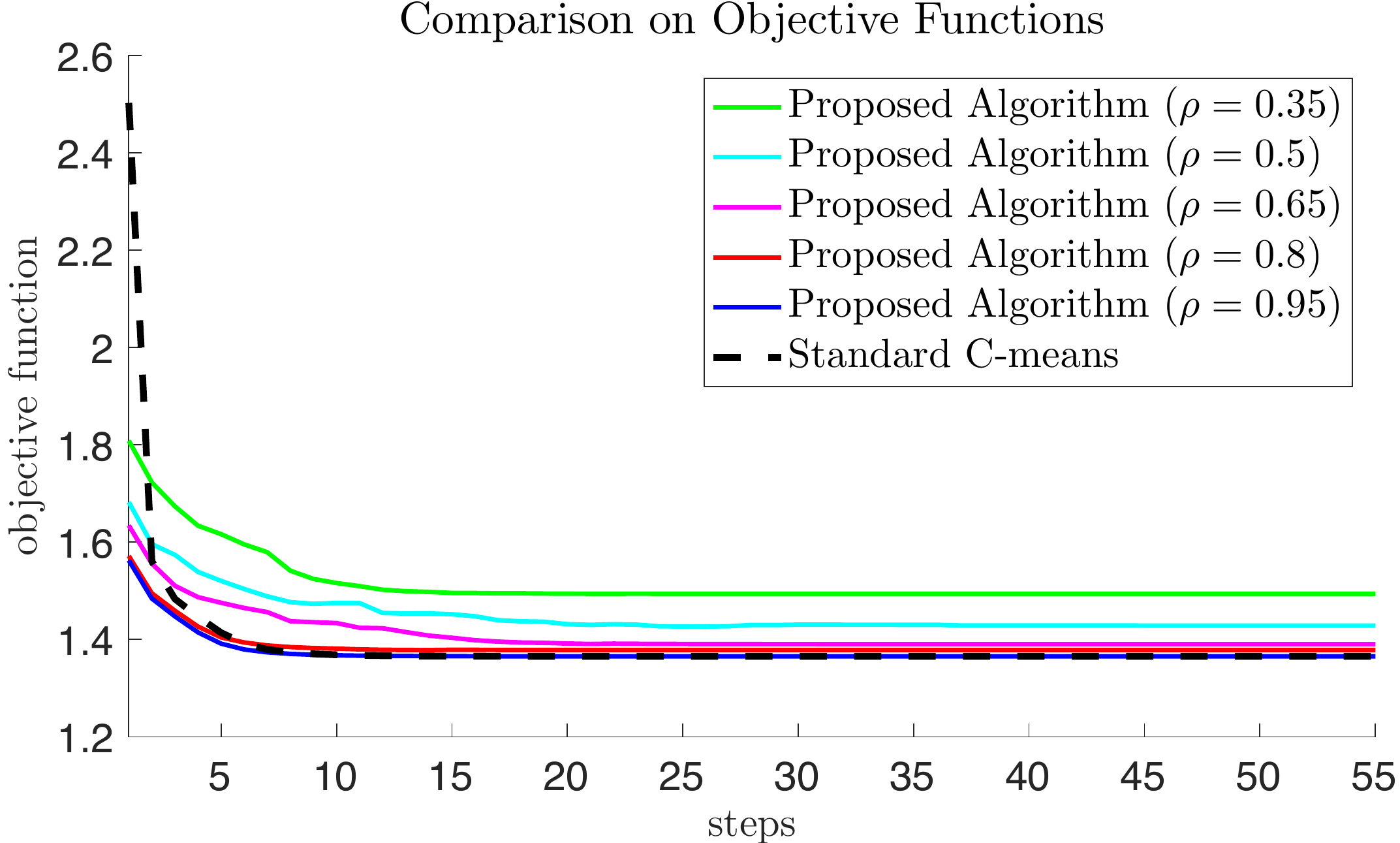} 

\caption{Comparison of the value of the objective function at each step of the proposed algorithm, for  different choices of $\rho$, with the one obtained by the standard C-means, considering the example of Figure \ref{fig:example1}.}
 \label{fig:example2}

\end{figure}


\bibliographystyle{IEEEtran}


\end{document}